%% file: main.tex
\documentclass[runningheads]{llncs}
\usepackage[T1]{fontenc}
\usepackage{graphicx}
\usepackage[textsize=small]{todonotes}

\usepackage{url}

\input{headers/commands}

\input{headers/notation}

\begin{document}
\mainmatter              %

\title{MAGICS: Adversarial RL with Minimax Actors Guided by Implicit Critic Stackelberg for Convergent Neural Synthesis of Robot Safety}
\titlerunning{Convergent Neural Synthesis of Robot Safety}  %
\author{Justin Wang\inst{1}\thanks{J.\ Wang and H. Hu contributed equally.} \and Haimin Hu\inst{1 \star} \and Duy P. Nguyen\inst{1} \and Jaime Fernández Fisac\inst{1}}
\authorrunning{Wang and Hu et al.} %
\institute{Department of Electrical and Computer Engineering, Princeton University, USA\\
\email{\{jw4406,haiminh,duyn,jfisac\}@princeton.edu}}

\maketitle              %

\begin{abstract}
While robust optimal control theory provides a rigorous framework to compute robot control policies that are provably safe, it struggles to scale to high-dimensional problems, leading to increased use of deep learning for tractable synthesis of robot safety.
Unfortunately, existing neural safety synthesis methods often lack convergence guarantees and solution interpretability.
In this paper, we present Minimax Actors Guided by Implicit Critic Stackelberg (MAGICS), 
a novel adversarial \gls{RL} algorithm that guarantees local convergence to a minimax equilibrium solution.
We then build on this approach to provide local convergence guarantees for a general deep \gls{RL}-based robot safety synthesis algorithm.
Through both simulation studies on OpenAI Gym environments and hardware experiments with a 36-dimensional quadruped robot, we show that MAGICS can yield robust control policies outperforming the state-of-the-art neural safety synthesis methods.

\end{abstract}

\section{Introduction}

The widespread deployment of autonomous robots calls for robust control methods to ensure their reliable operation in diverse environments.
Safety filters~\cite{hsu2023sf} have emerged as an effective approach to ensure safety for a broad spectrum of robotic systems, such as autonomous driving~\cite{zeng2021safety,tearle2021predictive,hu2022sharp,hu2023active}, legged locomotion~\cite{hsu2015control,agrawal2017discrete,he2024agile,nguyen2024gameplay}, and aerial robots~\cite{fisac2018general,singletary2022onboard,zhang2024gcbf+,chen2021fastrack}.
Model-based numerical safety synthesis methods~\cite{bansal2017hamilton,bui2021real} offer verifiable safety guarantees, but can only scale up to 5-6 state variables.
In order to develop safety filters at scale, recent research efforts have been dedicated towards using neural representations of robot policies~\cite{fisac2019bridging,robey2020learning,bansal2021deepreach,hsunguyen2023isaacs}, showing potential for safety filtering with tens~\cite{nguyen2024gameplay,he2024agile} even hundreds of state variables~\cite{hu2023deception}. 

\begin{figure}[!hbtp]
  \centering
  \includegraphics[width=1.0\columnwidth]{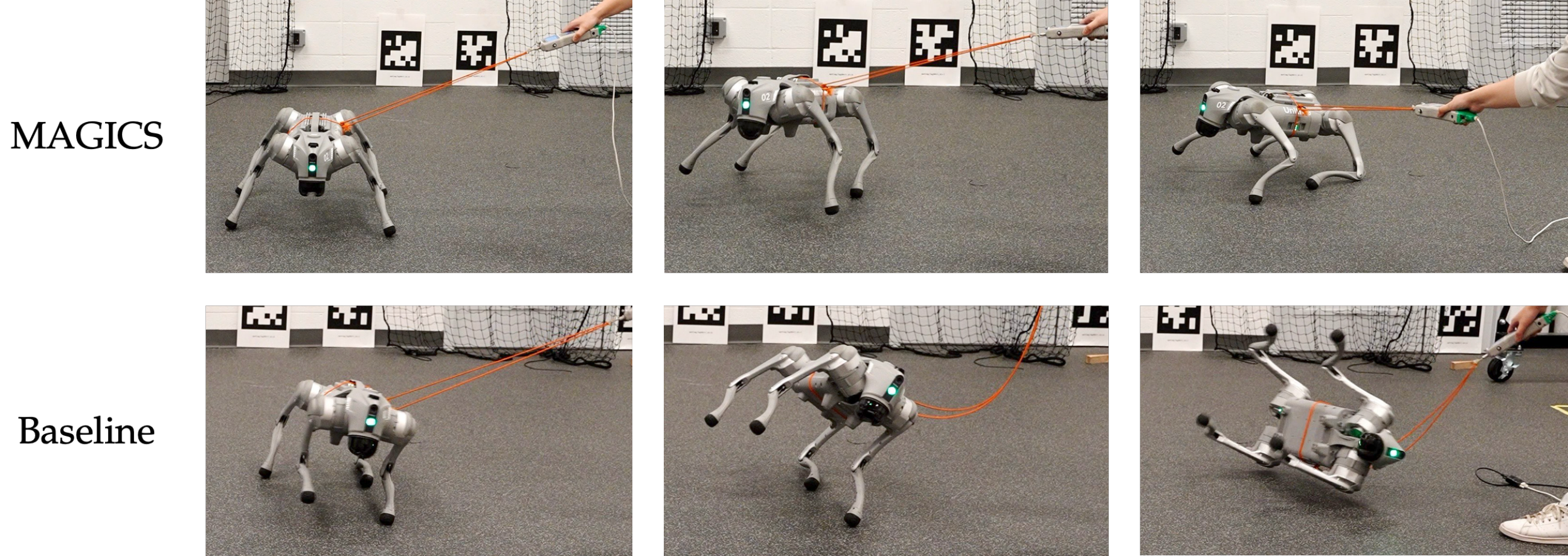}
  \caption{\label{fig:front_fig}Comparing to a non-game baseline robust \gls{RL}, our proposed game-theoretic adversarial RL algorithm yields a control policy consistently more robust when applied to safe quadrupedal locomotion and stress-tested with varying tugging forces.
  \vspace{-5mm}
  }
\end{figure}

While neural safety filters have shown significant progress in scalability, they are often challenging to analyze due to their black-box nature.
Moreover, in adversarial settings, na\"ive training of agent policies can lead to severe oscillatory behaviors, preventing the algorithm from converging.
Recently developed game-theoretic machine learning algorithms guarantee convergence to mathematically meaningful equilibrium solutions (\eg, Nash or Stackelberg), even for training of black-box models such as deep neural networks~\cite{fiez2020implicit,jin2020local,tijana2021who,zheng2022stackelberg}.
However, providing \textit{convergence guarantees} for training robust \textit{neural} control policies with zero-sum (adversarial) formulations remains an outstanding challenge.

\p{Statement of Contributions}
This paper introduces \gls{MAGICS}, a novel game-theoretic deep adversarial \gls{RL} algorithm that guarantees local convergence to a feedback Stackelberg equilibrium strategy.
Building on this framework, we propose, for the first time, a provably convergent neural synthesis algorithm for approximate robust safety analysis in high dimensions. 
In addition to our theoretical contributions, we demonstrate empirically that \gls{MAGICS} yields robust control policies consistently outperforming the prior state-of-the-art neural robust controller in both simulations and hardware tests.

\section{Related Work}

\p{Game-Theoretic Learning Algorithms}
Recent years have seen significant progress in machine learning problems formulated as games, such as generative adversarial networks (GANs)~\cite{goodfellow2020generative}, adversarial \gls{RL}~\cite{pinto2017robust}, and hyperparameter optimization~\cite{maclaurin2015gradient}.
Those problems involve interacting agents with coupled and potentially competing objectives, which calls for careful design of the training algorithm.
Fiez et al.~\cite{fiez2020implicit} leverages the implicit function theorem to derive a provably convergent gradient-based algorithm for machine learning problems formulated as Stackelberg games.
In their follow-up work~\cite{fiez2021local}, they show that the simple \gls{tGDA} algorithm guarantees local convergence to a minimax solution in zero-sum settings, a result also reported by concurrent work~\cite{jin2020local}.
More recent work~\cite{maheshwari2023convergent} shows that convergence can be achieved in Stackelberg learning only with first-order gradient information.

Our work builds on insights from the game-theoretic learning community and provides a novel, yet rigorous convergence analysis for adversarial \gls{RL}.

\p{Adversarial Reinforcement Learning}
Multi-agent robust policy synthesis is increasingly solved by adversarial \gls{RL} methods due to their scalability.
Pinto et al.~\cite{pinto2017robust} pioneered the idea of robust adversarial reinforcement learning (RARL), which improves robustness by jointly training a stabilizing policy and a destabilizing adversary.
However, this approach ignores the coupling between the agents' objectives.
Learning with Opponent-Learning Awareness (LOLA)~\cite{foerster2018learning} explicitly captures the coupled interests among agents in multi-agent \gls{RL} by accounting for the influence of one agent’s policy on the predicted parameter of the other agents.
Huang et al.~\cite{huang2022robust} model policy-gradient-based adversarial \gls{RL} as a Stackelberg game and apply the Stackelberg learning algorithm from Fiez et al.~\cite{fiez2020implicit} to solve the \gls{RL} problem.
Although these game-inspired adversarial \gls{RL} algorithms have shown promising improvement over non-game baselines, they typically lack provable convergence guarantees.
Recent work by Zheng et al.~\cite{zheng2022stackelberg} models the hierarchical interaction between the actor and critic in single-agent actor-critic-based \gls{RL}~\cite{konda1999actor,haarnoja2018soft} as a two-player general-sum game, and provides a gradient-based algorithm that is locally convergent to a Stackelberg equilibrium.

Our work combines the best of both worlds by building on the state-of-the-art game-theoretic \gls{RL}~\cite{zheng2022stackelberg} and adversarial learning~\cite{fiez2021local} approaches to address the missing piece of provably convergent \emph{multi-agent} adversarial \gls{RL} in continuous state--action spaces.

\p{Data-Driven Safety Filters for Robotics}
Deep learning has enabled scalable synthesis of robust control policies.
Rapid Motor Adaptation (RMA)~\cite{kumar2021rma} uses deep \gls{RL} to learn a base policy, and supervised learning to train an online adaptation module, allowing legged robots to quickly and robustly adapt to novel environments.
Specially focused on safety, data-driven safety filters~\cite{hsu2023sf,wabersich2023data} aim at providing scalable safety assurances.
Robey et al.~\cite{robey2020learning} use control barrier function (CBF) constraints as self-supervision signals to learn a CBF from expert demonstrations.
Another line of work focuses on neural approximation of safety Bellman equations.
Safety \gls{RL}~\cite{fisac2019bridging} proposes to modify Hamilton--Jacobi equation~\cite{bansal2017hamilton} with contraction mapping, rendering \gls{RL} suitable for approximate safety analysis.
Hsu et al.~\cite{hsunguyen2023isaacs} extend the single-agent safety \gls{RL} to the two-player zero-sum setting by offline co-training a best-effort safety controller and a worst-case disturbance policy.
However, existing multi-agent neural safety synthesis approaches predominantly lack convergence guarantees, which can make the training notoriously difficult and time-consuming.

To the best of our knowledge, our proposed game-theoretic deep adversarial \gls{RL} algorithm is the first \emph{provably-convergent} multi-agent neural safety synthesis approach.

\section{Preliminaries and Problem Formulation}

\p{Notation}
Given a function $f$, we denote $\total_\theta f$ as the total derivative of $f$ with respect to $\theta$, $\grad_\theta f$ as the partial derivative of $f$ with respect to $\theta$, $\grad_{\theta \psi} f := \partial^2 f / \partial \theta \partial \psi$, and $\grad_\theta^2 f := \partial^2 f / \partial \theta^2$.
We denote $\|\cdot\|$ as the 2-norm of vectors and the spectral norm of matrices.
We indicate matrix $A$ is positive and negative definite with $A \succ 0$ and $A \prec 0$.
We assume throughout that all functions $f$ are smooth, \ie, $f \in C^q$ for some $q \geq 2$.

\p{Zero-Sum Dynamic Games}
We consider discrete-time nonlinear dynamics that describe the robot motion:
\begin{equation}
\label{eq:dyn_sys}
    \state_{t+1} = \dyn (\state_t, \ctrl_t, \dstb_t),
\end{equation}
where $\state_t \in \xset\subseteq\reals^{\nx}$ is the state,
$\ctrl_t \in \cset \subset \reals^{\nc}$ is the controller input (belongs to the ego robot),
$\dstb_t \in \dset \subset \reals^{\nd}$ is the disturbance input,
and $\dyn: \xset \times \cset \times \dset \to \xset$ is a continuous nonlinear function that describes the physical system dynamics.
We model the \textit{adversarial} interaction between the controller and disturbance as a \textit{zero-sum dynamic game} with objective $\cost^{\policyc, \policyd}(\traj)$, where $\policyc,\policyd: \xset \rightarrow \cset$ are control and disturbance policies, respectively, $\traj := \traj^{\policyc, \policyd}_{\state, T} = (\state_0, \state_1, \ldots, \state_{T-1})$ denotes the state trajectory starting from $\state_0 = \state$ under dynamics~\eqref{eq:dyn_sys}, policies $\policyc$, and $\policyd$.
In this paper, we seek to compute a control policy $\policyc$ and disturbance policy $\policyd$ that constitute a \textit{\gls{LSE}} in the space of feedback policies of a zero-sum dynamic game, where $\policyc$ maximizes $\cost^{\policyc, \policyd}(\traj)$ and $\policyd$ minimizes $\cost^{\policyc, \policyd}(\traj)$. This is formalized in the following definition.

\begin{definition}[Local Stackelberg Equilibrium]
\label{def:LSE}
    Let $\Pi^\ctrl$ and $\Pi^\dstb$ be the sets of control and disturbance policies, respectively. The strategy pair $(\policy^{\ctrl,*}, \policy^{\dstb,*}) \in \Pi^\ctrl \times \Pi^\dstb$ is a local Stackelberg equilibrium if
    \begin{equation}
        \inf_{\policy^\dstb \in \resmap(\policy^{\ctrl,*})} \cost^{\policy^{\ctrl,*}, \policy^\dstb}(\traj) \geq \inf_{\policy^\dstb \in \resmap(\tilde{\policy}^\ctrl)} \cost^{\tilde{\policy}^\ctrl, \policy^\dstb}(\traj),\quad \forall \tilde{\policy}^\ctrl \in \Pi^\ctrl,
    \end{equation}
    and $\policy^{b,*} \in \resmap(\policy^{\ctrl,*})$, where $\resmap(\policy^\ctrl) := \{\policy^\dstb \in \Pi^\dstb \mid \cost^{\policy^\ctrl, \tilde{\policy}^\dstb} \geq \cost^{\policy^\ctrl, \policy^\dstb},~\forall \tilde{\policy}^\dstb \in \Pi^\dstb\}$ is the optimal response map of the follower.
    Here, $\Pi^\ctrl$ and $\Pi^\dstb$ are the sets of control and disturbance policies, respectively.
\end{definition}

We remark that an LSE in the space of feedback policies is, in general, \emph{not} the same as an \gls{FSE}.
However, the two types of equilibria \emph{can} coincide in some special circumstances as outlined in \cite{2021_FSE_LSE}.

\autoref{def:LSE} is aligned with our interest in \emph{worst-case} robot safety analysis: by assigning the controller as the leader and disturbance the follower, we give the disturbance the \textit{instantaneous information advantage}~\cite{isaacs1954differential}.

An \gls{LSE} can be characterized using the first- and second-order optimality conditions (sufficient conditions of those in~\autoref{def:LSE}), leading to the notion of a \gls{DSE}~\cite{fiez2020implicit}, which can be verified more easily.
To this end, we adopt a finite-dimensional parameterization of players' policies: $\policyc = \policyc_\policycparam$ and $\policyd = \policyd_\policycparam$, where $\policycparam \in \reals^{n_\policycparam}$ and $\policydparam \in \reals^{n_\policydparam}$.
In \autoref{sec:approach}, we consider the neural representation of policies, \ie, $\policyc_\policycparam$ and $\policyd_\policycparam$ are deep neural networks.
We denote the resulting game objective value as $\cost^{\policyc_\policycparam, \policyd_\policydparam}(\policycparam, \policydparam, \traj)$.
We recall that, in zero-sum games, a \gls{DSE} is equivalent to a local minimax equilibrium.

\begin{definition}[Strict Local Minimax Equilibrium~\cite{fiez2020implicit,fiez2021global,jin2020local}]
Strategy pair $(\policyc_\policycparameq, \policyd_\policydparameq)$ is a differential Stackelberg equilibrium of a zero-sum dynamic game if conditions $\grad_\policycparam \cost^{\policy^{\ctrl}_\policycparameq, \policy^\dstb_\policydparameq}(\policycparameq, \policydparameq, \traj) = 0$, $\grad_\policydparam \cost^{\policy^{\ctrl}_\policycparameq, \policy^\dstb_\policydparameq}(\policycparameq, \policydparameq, \traj) = 0$, $\schur(\jacobian_\cost(\policycparameq, \policydparameq, \traj)) \prec 0$, and $\grad^2_\policydparam(\jacobian_\cost(\policycparameq, \policydparameq, \traj)) \succ 0$ hold, where $\jacobian_\cost(\policycparam, \policydparam, \traj)$ denotes the Jacobian of the \emph{individual} gradient vector:
\vspace{-2mm}
\begin{equation}
    \jacobian_\cost(\policycparam, \policydparam, \traj) = 
    \begin{bmatrix}
    -\nabla_\policycparam^2 \cost^{\policy^{\ctrl}_\policycparam, \policy^\dstb_\policydparam}(\policycparam, \policydparam, \traj)           &\quad~-\nabla_{\policycparam\policydparam} \cost^{\policy^{\ctrl}_\policycparam, \policy^\dstb_\policydparam}(\policycparam, \policydparam, \traj) \\
    \nabla_{\policycparam\policydparam}^{\top} \cost^{\policy^{\ctrl}_\policycparam, \policy^\dstb_\policydparam}(\policycparam, \policydparam, \traj)  &\quad~\nabla_\policydparam^2 \cost^{\policy^{\ctrl}_\policycparam, \policy^\dstb_\policydparam}(\policycparam, \policydparam, \traj)
    \end{bmatrix},
    \vspace{-1mm}
\end{equation}
and $\schur(\jacobian_\cost)$ denotes the Schur complement of $\jacobian_\cost$ with respect to its lower-right block.
\end{definition}

\p{Reach--Avoid Robot Safety Analysis}
We consider the \emph{worst-case} reach--avoid safety analysis for robot dynamics~\eqref{eq:dyn_sys}.
We define the ego robot's target and failure sets as ${\target := \{\state \mid \ell(\state) \geq 0\} \subseteq \reals^{\nx}}$ and ${\failure := \{\state \mid g(\state) < 0\} \subseteq \reals^{\nx}}$, where $\ell(\cdot)$ and $g(\cdot)$ are Lipschitz continuous \emph{margin functions}, encoding problem-specific safety and liveness specifications.
We use \gls{HJI} reachability analysis to capture the interplay between the best-effort controller policy $\policyc$, \ie, one that attempts to reach the target set $\target$ without entering the failure set $\failure$, and worst-case disturbance policy $\policyd$, \ie, one that prevents the controller from succeeding.
To this end, we formulate an infinite-horizon \textit{zero-sum} dynamic game with the following objective functional:
\begin{equation}
\begin{aligned}
\label{eq:zs_game}
    \cost_k^{\policyc, \policyd}(\state) := \max_{\tau \geq k} \min \left\{ \ell(\state_\tau), \min_{s \in [k, \tau]} g(\state_s) \right\}.
\end{aligned}
\end{equation}
The game's minimax solution satisfies the fixed-point Isaacs equation~\cite{isaacs1954differential}:
\begin{equation}
\label{eq:Isaacs_basic}
\begin{aligned}
    \valfunc(\state) = \min \left\{ g(\state), \max \Big\{ \ell(\state),  \max_{\ctrl \in \cset} \min_{\dstb \in \dset} \valfunc \big(\dyn (\state, \ctrl, \dstb)\big) \Big\}  \right\}.
\end{aligned}
\end{equation}
Value function $\valfunc(\cdot)$ encodes the \textit{reach--avoid set} $\reachavoid(\target, \failure) := \{\state \mid \valfunc(\state) \geq 0\}$, from which the ego agent is guaranteed a policy to safely reach the target set without entering the failure set.
Note that \eqref{eq:Isaacs_basic} implies that the disturbance has the instantaneous informational advantage and $\valfunc(\state) = \max_{\policyc} \min_{\policyd} \cost^{\policyc, \policyd}(\state)$ is known as the \textit{lower} value of the zero-sum game~\eqref{eq:zs_game}.
It corresponds to the minimal safety margin $g$ that the controller is able to maintain \emph{at all times} under the worst-case disturbance.

\section{Approach: Stackelberg--Minimax Adversarial RL}
\label{sec:approach}

\p{Soft Adversarial Actor--Critic as a Three-Player Game}
We consider a discrete-time adversarial Markov game governed by system~\eqref{eq:dyn_sys}.
The initial state is determined by a given prior distribution $\state_0 \sim \distr_0(\state)$.
At time $t$, the controller and disturbance take actions according to their stochastic policies, \ie, $\ctrl_t \sim \policyc_\policycparam(\cdot | \state_t)$, $\dstb_t \sim \policyd_\policydparam(\cdot | \state_t)$, the controller receives a bounded reward $\reward_t = \reward(\state_t, \ctrl_t, \dstb_t)$ emitted from the environment, and the disturbance receives reward $-\reward_t$.
The single-player version of this Markov game (\ie, with $\dstb_t \equiv 0$) can be solved at scale using deep \gls{RL} approaches. We will focus on both \emph{on- and off-policy actor-critic} methods; specifically \gls{A2C}~\cite{a2c_2016} and \gls{SAC} ~\cite{haarnoja2018soft}, respectively.
In the following, we present adversarial variants of the \gls{A2C} and \gls{SAC} algorithms.
We assume that the critic and actors are deep neural networks parameterized by $(\criticparam, \policycparam, \policydparam)$.
The critic aims to minimize the mean-square Bellman error loss
\begin{equation}
    \label{eq:a2c_critic}
    \criticcost(\criticparam,\policycparam,\policydparam) = \expectation_{\state \sim \traj} \left[ (\valfunc_\criticparam(\state) - \valfunc^\pi(\state))^2\right],
\end{equation}
or 
\begin{equation}
    \label{eq:sac_critic}
    \criticcost(\criticparam,\policycparam,\policydparam) = \expectation_{\xi \sim \buffer} \left[ \left(\qfunc_{\criticparam_1} (\state,\ctrl,\dstb) - \reward - \discount \qfunc_{\criticparam_2}(\state^\prime, \ctrl^\prime, \dstb^\prime) \right)^2\right],
\end{equation}
for \gls{A2C} and \gls{SAC}, respectively, in which $\valfunc$ is the value function, $\qfunc$ is the state--action value function, $\criticparam = (\criticparam_1, \criticparam_2) \in \reals^{n_\criticparam}$ is the critic's parameter, $\xi = (\state,\ctrl,\dstb,\reward,\state^\prime)$ is the transition data, $\buffer$ is a replay buffer, $\discount \in (0,1]$ is a discount factor, $\ctrl^\prime \sim \policyc_{\policycparam} (\cdot | \state^\prime)$, $\dstb^\prime \sim \policyd_{\policydparam} (\cdot | \state^\prime)$, and the function $\valfunc^\pi(\state)$ is approximated through a bootstrapped estimator (here we use generalized advantage estimation \cite{gae_lambda}).
For \gls{A2C}, the controller seeks to maximize the objective 
\begin{equation}
    \label{eq:a2c_actor}
    \cost(\criticparam, \policycparam, \policydparam) = \expectation_{\traj \sim (\policyc, \policyd)}\left[r(\state_0, \ctrl_0, \dstb_0) + \valfunc_\criticparam(\state_1)\right],
\end{equation}
while for \gls{SAC}, the controller seeks to maximize entropy-regularized control objective
\begin{equation}
\begin{aligned}
    \label{eq:sac_actor}
    &\cost(\criticparam,\policycparam,\policydparam) = \\
    &\expectation_{\state \sim \buffer} \left[ \min_{i \in \{1,2\}} \qfunc_{\criticparam_i} (\state,\tilde{\ctrl},\tilde{\dstb})
    - \entreg^\ctrl \left[\log \policyc_\policycparam(\tilde{\ctrl} | \state) - H_0^\ctrl\right]
    + \entreg^\dstb \left[\log \policyd_\policydparam(\tilde{\dstb} | \state) - H_0^\dstb\right]\right],
\end{aligned}
\end{equation}
where $\tilde{\ctrl} \sim \policyc_{\policycparam} (\cdot | \state)$, $\tilde{\dstb} \sim \policyd_{\policydparam} (\cdot | \state)$, $\entreg^\ctrl, \entreg^\dstb > 0$ are the entropy regularization coefficients, and $H_0^\ctrl$ and $H_0^\dstb$ are the minimum entropy (heuristically set to the dimension of the player action space, following \cite{haarnoja_sac_auto_temp}) of controller and disturbance policies, respectively.
Finally, the disturbance minimizes $\cost(\criticparam,\policycparam,\policydparam)$.

We cast \gls{MAGICS} training procedure defined by \eqref{eq:sac_critic}-\eqref{eq:sac_actor} as a \emph{three-player} general-sum (\ie, non-cooperative) Stackelberg game, which is modeled by the following \textit{trilevel} optimization problem:
\begin{subequations}
\label{eq:trilevel}
\begin{align}
    \min_{\criticparam} \quad &\criticcost(\criticparam,\policycparam,\policydparam) \label{eq:trilevel_critic}\\
    \text{s.t.} \quad~&\policycparam \in \textstyle\argmax_{\tilde{\policycparam}} \cost(\criticparam, \tilde{\policycparam},\bar{\policydparam}), \label{eq:trilevel_ctrl} \\
    &\text{s.t.} \quad \bar{\policydparam} \in \textstyle\argmin_{{\policydparam^\prime}} \cost(\criticparam, \policycparam,{\policydparam^\prime}), \label{eq:trilevel_dstb}
\end{align}
\end{subequations}
where the critic is the leader, followed by the controller, and the disturbance plays last, consistent with its information advantage encoded in~\eqref{eq:Isaacs_basic}. The most natural way of viewing the three-player game is as follows: in line with \cite{fiez2021global}, the two primary players are the controller $\ctrl$ and disturbance $\dstb$, while the critic sits in judgment of their gameplay.

\p{Solving the Zero-Sum Actor Game with $\tau$-GDA}
In this section, we provide a subroutine that solves the inner zero-sum game between the two actors \eqref{eq:trilevel_ctrl}-\eqref{eq:trilevel_dstb} using \gls{tGDA}~\cite{fiez2021global,fiez2021local} for a fixed critic parameter $\criticparam$.
In essence, this approach scales up the learning rate of the follower so that it can adapt faster, which leads to guaranteed local convergence to a minimax solution of the zero-sum game~\eqref{eq:trilevel_ctrl}-\eqref{eq:trilevel_dstb}.
This subroutine is summarized in Algorithm~\ref{alg:tau_gda}.

\begin{remark}
    While we focus on the local convergence analysis for game-theoretic adversarial \gls{RL} involving optimization of neural network parameters, recent work~\cite{fiez2021global} shows that, in a zero-sum game, if the maximizing player has a Polyak-Łojasiewicz (PŁ) or strongly-concave (SC) objective, then \gls{tGDA} can converge globally to a strict local minimax equilibrium.
\end{remark}

\begin{algorithm}[!ht]
\DontPrintSemicolon
\small
\caption{\gls{tGDA} for Zero-Sum Actor Game}\label{alg:tau_gda}
\SetKwInput{KwRequire}{Input}
\KwRequire{actor parameters $\left(\policycparam_0, \policydparam_0\right)$, critic parameter $\criticparam$, controller learning rate schedule $\{\lr^\ctrl_{i}\}_{i=0,1,\ldots}$, learning rate ratio $\tau_a$}
$i \gets 0$\;
\While{\textit{Not converged}}{
    $\policycparam_{i+1} \gets \policycparam_i + \lr^\ctrl_{i} \grad_\policycparam \cost(\criticparam, \policycparam_i, \policydparam_i)$\tcp*{Updates controller parameters}
    $\policydparam_{i+1} \gets \policydparam_i - \lr^\ctrl_{i} \tau_a \grad_\policydparam \cost(\criticparam, \policycparam_i, \policydparam_i)$\tcp*{Updates disturbance parameters}
    $i \gets i + 1$\;
}
\end{algorithm}

\p{MAGICS: Minimax Actors Guided by Implicit Critic Stackelberg}
We now return to the full game~\eqref{eq:trilevel}. Since the controller and the disturbance update their policy parameters using \gls{tGDA}, the critic sees them as two \emph{simultaneous} (\ie, Nash) followers: the controller performs gradient ascent on return $\cost(\criticparam, \policycparam, \policydparam)$, and the disturbance performs gradient descent on $\tau_a \cost(\criticparam, \policycparam, \policydparam)$.
In the following, we derive the update rule for the \gls{MAGICS}.
We start by providing the critic's Stackelberg learning dynamics~\cite{fiez2020implicit}.
Invoking the implicit function theorem, the total derivative of the critic's cost $\total_\criticparam \criticcost(\criticparam, \policycparam^*, \policydparam^*)$ at a minimax equilibrium $(\policycparam^*, \policydparam^*)$ is:
\begin{equation*}
    \total_\criticparam \criticcost(\criticparam, \policycparam^*, \policydparam^*) = \grad_\criticparam \criticcost(\criticparam) 
    +  \grad_\criticparam \policycparam^*(\criticparam) \grad_\policycparam \criticcost(\criticparam, \policycparam^*, \policydparam^*)
    +  \grad_\criticparam \policydparam^*(\criticparam) \grad_\policydparam \criticcost(\criticparam, \policycparam^*, \policydparam^*),
\end{equation*}
where $\policycparam^*(\criticparam) = \resmap_\policycparam(\criticparam)$ and $\policydparam^*(\criticparam) = \resmap_\policydparam(\criticparam)$ are the controller's and disturbance's rational response to the critic, respectively. 
The implicit differentiation terms $\grad_\criticparam \policycparam^*(\criticparam)$ and $\grad_\criticparam \policydparam^*(\criticparam)$ can be computed by solving the following linear system of equations:
\begin{equation*}
    \begin{aligned}
    \left\{\begin{array}{l}
    0 = \grad_{\policycparam\criticparam} \cost(\criticparam, \policycparam^*, \policydparam^*) + \grad_{\policycparam}^2 \cost(\criticparam, \policycparam^*, \policydparam^*) \grad_\criticparam \policycparam^*(\criticparam) + \grad_{\policycparam\policydparam} \cost(\criticparam, \policycparam^*, \policydparam^*) \grad_\criticparam \policydparam^*(\criticparam), \\[0.1cm]
    0 = \grad_{\policydparam\criticparam} \cost(\criticparam, \policycparam^*, \policydparam^*) + \grad_{\policydparam\policycparam} \cost(\criticparam, \policycparam^*, \policydparam^*) \grad_\criticparam \policycparam^*(\criticparam) + \grad^2_{\policydparam} \cost(\criticparam, \policycparam^*, \policydparam^*) \grad_\criticparam \policydparam^*(\criticparam).
    \end{array}\right.
    \end{aligned}
\end{equation*}
Compactly, the total derivative of the critic can be written as:
\begin{equation}
\label{eq:critic_total_deriv_compact}
    \total_\criticparam \criticcost(\criticparam, \policycparam, \policydparam) = \grad_\criticparam \criticcost(\criticparam, \policycparam, \policydparam) 
    -  h_1(\criticparam, \policycparam, \policydparam)^\top H(\criticparam, \policycparam, \policydparam)^{-1} h_2(\criticparam, \policycparam, \policydparam),
\end{equation}
where
\begin{equation*}
    h_1(\criticparam, \policycparam, \policydparam) =
    \begin{bmatrix}
        \grad_{\criticparam\policycparam} \cost(\criticparam, \policycparam, \policydparam) \\
        \grad_{\criticparam\policydparam} \cost(\criticparam, \policycparam, \policydparam)
    \end{bmatrix}, \quad
    h_2(\criticparam, \policycparam, \policydparam) =
    \begin{bmatrix}
    \grad_\policycparam \criticcost(\criticparam, \policycparam, \policydparam) \\
    \grad_\policydparam \criticcost(\criticparam, \policycparam, \policydparam)
    \end{bmatrix},
\end{equation*}
and
\begin{equation*}
    H(\criticparam, \policycparam, \policydparam) =
    \begin{bmatrix}
        \grad_{\policycparam}^2 \cost(\criticparam, \policycparam, \policydparam) & \quad \grad_{\policycparam\policydparam} \cost(\criticparam, \policycparam, \policydparam) \\
        \grad_{\policydparam\policycparam} \cost(\criticparam, \policycparam, \policydparam) & \quad \grad^2_{\policydparam} \cost(\criticparam, \policycparam, \policydparam)
    \end{bmatrix}.
\end{equation*}
Each term in the Stackelberg gradient update rule can be computed directly and estimated by samples. 

For MAGICS-\gls{A2C}, the objective functions are defined in expectation over the distribution induced by the players' current policies. Hence, custom gradient estimators are required, as the Stackelberg gradient of \eqref{eq:a2c_critic} is not straightforward. The estimator is given below, and the proof merely requires expanding the value and state--action value function definitions recursively.

\begin{algorithm}[!ht]
\DontPrintSemicolon
\small
\caption{Stackelberg--Minimax Adversarial RL}\label{alg:MAGICS}
\SetKwInput{KwRequire}{Input}
\KwRequire{parameters $\left(\criticparam_0, \policycparam_0, \policydparam_0\right)$, critic learning rate schedule $\{\lr^c_i\}_{i=0,1,\ldots}$, controller learning rate schedule $\{\lr^\ctrl_i\}_{i=0,1,\ldots}$, actor learning rate ratio $\tau_a$}
$i \gets 0$\;
\While{\textit{Not converged}}{
    $\criticparam_{i+1} \gets \criticparam_i - \lr^c_i \total_\criticparam \criticcost(\criticparam_i, \policycparam_i, \policydparam_i)$\tcp*{Updates critic parameters}
    $(\policycparam_{i+1}, \policydparam_{i+1}) \gets$ \gls{tGDA}$(\criticparam_i, \policycparam_i, \policydparam_i, \lr^\ctrl_i, \tau_a)$\tcp*{Updates actor parameters}
    $i \gets i + 1$\;
}
\end{algorithm}

\begin{theorem}
\label{thm:critic_partial}
    Given a Markov game with actor parameters $(\policycparam, \policydparam)$ and shared critic parameters $\criticparam$, if critic has objective function $\criticcost(\criticparam, \policycparam, \policydparam)$ defined in \eqref{eq:a2c_critic}, then $\grad_\policycparam \criticcost(\criticparam, \policycparam, \policydparam)$ is given by 
    \begin{equation}
    \begin{aligned}
        \grad_\policycparam \criticcost(&\criticparam, \policycparam, \policydparam) = \\
        &\expectation_{\traj \sim (\policycparam, \policydparam)} \left[2 \sum_{t=0}^T \discount^t \grad_\policycparam \log \policyc_\policycparam(\ctrl_t | \state_t) \left(\valfunc_\criticparam(\state_0) - \valfunc^\policy(\state_0)\right)\qfunc^\policy(\state_t, \ctrl_t, \dstb_t)\right].
    \end{aligned}
    \end{equation}
\end{theorem}

MAGICS-\gls{SAC}, on the other hand, is an off-policy \gls{RL} scheme, with critic's and actors' objective defined as an expectation over an arbitrary distribution from a replay buffer. An unbiased estimator for each term in the Stackelberg gradient can be computed directly from samples using automatic differentiation, \eg,  
\begin{equation}
\label{eq:unbias}
\begin{aligned}
\grad_\criticparam \criticcost(\criticparam, \policycparam, \policydparam) 
&= \expectation_{\xi \sim \buffer} \left[ \grad_\criticparam \left(\left(\qfunc_{\criticparam_1} (\state,\ctrl,\dstb) - \reward - \discount \qfunc_{\criticparam_2}(\state^\prime, \ctrl^\prime, \dstb^\prime) \right)^2\right)\right] \\
&\approx \frac{1}{N} \sum_{n=1}^N \grad_\criticparam \left(\left(\qfunc_{\criticparam_1} (\state_{n},\ctrl_{n},\dstb_{n}) - \reward_{n} - \discount \qfunc_{\criticparam_2}(\state^\prime_{n}, \ctrl^\prime_{n}, \dstb^\prime_{n}) \right)^2\right).
\end{aligned}
\end{equation}
For the second term in \eqref{eq:critic_total_deriv_compact}, we can estimate each term $h_1(\cdot)$, $H(\cdot)$, and $h_2(\cdot)$ individually using samples from the replay buffer and resetting the simulator~\cite[Chapter 11]{sutton2018reinforcement}.
In order to numerically determine convergence, we may check the magnitude of the gradient norms, \ie, $\|\total_\criticparam \criticcost(\cdot)\| \leq \epsilon^c$, $\|\grad_\policycparam \cost(\cdot)\| \leq \epsilon^\ctrl$, $\|\grad_\policydparam \cost(\cdot)\| \leq \epsilon^\dstb$, where $\epsilon^c$, $\epsilon^\ctrl$, $\epsilon^\dstb$ are small thresholds.
The overall procedure of \gls{MAGICS} is summarized in Algorithm~\ref{alg:MAGICS}.

\p{Complexity and Gradient Computation Details}
The critic's Stackelberg gradient requires an \gls{iHvp} and a \gls{jvp}. The latter can be computed directly through a call to \texttt{torch}'s automatic differentiation engine \texttt{autograd.grad}. The Hessian $H$ in~\eqref{eq:critic_total_deriv_compact} is a 2 $\times$ 2 block matrix composed of differentiating the loss functions of the critic and actor(s) against their respective parameters; accordingly, implicit representations using Hessian-vector products cannot be used. If $h_1$'s differentiation against the player parameters is done first, assembling $H$ requires three additional \texttt{autograd.grad} calls---one for each of the block-diagonal elements as well as the $(1,2)$-entry due to symmetry. $h_2$ can be assembled straightforwardly with two calls to \texttt{autograd.grad}, and the \gls{iHvp} can be computed using \texttt{torch}'s built-in linear systems solver \texttt{linalg.solve}. The \gls{jvp} between $h_1$ and the \gls{iHvp} can be computed in one final \texttt{autograd.grad} for a total of eight calls. For MAGICS-\gls{A2C}, the sequence of eight \texttt{autograd} calls is not prohibitively slow because of the custom gradient estimators---the closed-form estimator with the direct log-likelihood structure ensures that the differentiation does not need to traverse a large computational graph. 
In contrast, the eight \texttt{autograd} calls require massive computational expenditure for MAGICS-\gls{SAC}.
To enhance computational efficiency, in our implementation of MAGICS-\gls{SAC}, we use an empirical Fisher approximation to the Hessian \cite{ef1,ef2} for computing~\eqref{eq:critic_total_deriv_compact}.
The direct log-likelihood structure and Einstein summation of rank-one matrices can be computed in time of a similar order of magnitude to that of MAGICS-\gls{A2C}, which ameliorates the computational burden.

For MAGICS-\gls{A2C}, each Stackelberg gradient step can be computed in $\sim5.5$  times the cost of a baseline gradient step.
For MAGICS-\gls{SAC}, the full Stackelberg gradient incurs $\sim10.5$ times the computational cost of the baseline gradient, whereas the empirical Fisher implementation significantly reduces this, requiring only about $\sim1.8$ times the baseline cost.
We find empirically the increase in computation by incorporating the Stackelberg gradient in MAGICS-\gls{A2C} and MAGICS-\gls{SAC} (with empirical Fisher approximation) to be bearable even for high-dimensional systems.

\p{Convergence Analysis of \gls{MAGICS}}
Due to sample-based approximation of gradients, the \gls{MAGICS} procedure can be modeled by the discrete-time dynamical system:
\begin{subequations}
\label{eq:magics_dt_sys}
\begin{align}
    \criticparam_{t+1} &= \criticparam_t - \lr^c_t \left(\total_\criticparam \criticcost(\criticparam_t, \policycparam_t, \policydparam_t) + \noise_{\criticparam,t} \right), \label{eq:magics_dt_sys:critic} \\
    \policycparam_{t+1} &= \policycparam_t + \lr^\ctrl_t \left(\grad_\policycparam \cost(\criticparam_{t+1}, \policycparam_t, \policydparam_t) + \noise_{\policycparam,t} \right), \label{eq:magics_dt_sys:ctrl} \\
    \policydparam_{t+1} &= \policydparam_t - \tau_a \lr^\ctrl_t \left(\grad_\policydparam \cost(\criticparam_{t+1}, \policycparam_t, \policydparam_t) + \noise_{\policydparam,t} \right). \label{eq:magics_dt_sys:dstb}
\end{align}
\end{subequations}
In the next, we show that the \gls{MAGICS} procedure, modeled by system~\eqref{eq:magics_dt_sys}, locally converges to a game-theoretically meaningful equilibrium solution.
We start by showing that \gls{tGDA} can \textit{robustly} converge to a minimax equilibrium when the critic parameter $\criticparam$ varies within a local region, \ie, $\criticparam \in \region_{\bar{\criticparam}} := \{\tilde{\criticparam} \mid \|\tilde{\criticparam} - \bar{\criticparam}\| < \epsilon_\criticparam \}$.
This is formalized in the following Lemma, which extends the two-player \gls{tGDA} local convergence result~\cite[Theorem~1]{fiez2021local} to additionally account for a ``meta'' leader (\ie, the critic).

\begin{assumption}
\label{assump:main}
We assume that the following hold.
    \begin{enumerate}[(a)]
        \item The maps $\total_\criticparam \criticcost(\cdot)$, $\grad_\policycparam \cost(\cdot)$, $\grad_\policydparam \cost(\cdot)$ are $L_1$, $L_2$, $L_3$ Lipschitz, and $\|\total_\criticparam \criticcost\| < \infty$.
        \item The critic learning rates are square summable but not summable, \ie, $\sum_{k} \lr^i_k = \infty$, $\sum_{k} (\lr^i_k)^2 < \infty$ for $i \in \{c, \ctrl\}$.
        \item The noise processes $\{\noise_{\criticparam,t}\}$, $\{\noise_{\policycparam,t}\}$, and $\{\noise_{\policydparam,t}\}$ are zero-mean, martingale difference sequences (c.f.~\cite[Assumption~1]{zheng2022stackelberg}).
    \end{enumerate}
\end{assumption}

\begin{lemma}[Robust Stability of DSE under \gls{tGDA}]
\label{lem:tGDA_stab}
Consider the zero-sum game between the controller and disturbance  $(-\cost, \cost)$ parameterized by $\criticparam \in \region_{\bar{\criticparam}} = \{\tilde{\criticparam} \mid \|\tilde{\criticparam} - \bar{\criticparam}\| < \epsilon_\criticparam \}$.
Let \autoref{assump:main} hold.
If $\jpolicyparam^* := (\policycparam^*,\policydparam^*)$ is a \gls{DSE} (minimax equilibrium), then there exists a $\tau^* \in (0, \infty)$ such that, for all $\tau \in (\tau^*, \infty)$ and for all $\criticparam \in \region_{\bar{\criticparam}}$, $\{y_k\}$ almost surely converges locally asymptotically to $\jpolicyparam^*$ 
\end{lemma}

\begin{proof}
    Given critic parameter $\bar{\criticparam}$, the continuous-time limiting system of the noise-free \gls{tGDA} updates~\eqref{eq:magics_dt_sys:ctrl}-\eqref{eq:magics_dt_sys:dstb} is $\dot{\jpolicyparam} = (\grad_\policycparam \cost(\bar{\criticparam}, \jpolicyparam), - \tau_a \grad_\policydparam \cost(\bar{\criticparam}, \jpolicyparam))$, with Jacobian denoted as $\jacobian_{\tau_a}(\jpolicyparam; \bar{\criticparam})$.
    By~\cite[Theorem~1]{fiez2021local}, it is possible to explicitly construct $\tau_{\bar{\criticparam}} = \lambda^+_{\max}(Q_{\bar{\criticparam}}(\jacobian_{\tau_a}(\jpolicyparam^*; \bar{\criticparam})))$, such that $\jacobian_{\tau_a}(\jpolicyparam^*; \bar{\criticparam})$ is Hurwitz, \ie, $\jpolicyparam^*$ is locally exponentially stable, for all $\tau \in (\tau_{\bar{\criticparam}}, \infty)$.
    Here, matrix $Q_{\bar{\criticparam}}$ is computed based on blocks of Jacobian $\jacobian_{\tau_a}(\jpolicyparam^*; \bar{\criticparam})$, and is $L_{\bar{\criticparam}}$ Lipschitz continuous in $\criticparam$.
    Next, we extend this result to construct a finite learning rate ratio such that $\jacobian_{\tau_a}(\jpolicyparam^*; \criticparam)$ is Hurwitz for all $\criticparam \in \region_{\bar{\criticparam}}$.
    For any $\criticparam \in \region_{\bar{\criticparam}}$, define $\Delta Q_{\criticparam} = Q_{\criticparam} - Q_{\bar{\criticparam}}$.
    By Weyl's inequality, we have $\lambda^+_{\max}(Q_{\criticparam}) \leq \lambda^+_{\max}(Q_{\bar{\criticparam}}) + \|\Delta Q_{\criticparam}\| \leq \lambda^+_{\max}(Q_{\bar{\criticparam}}) + L_{\bar{\criticparam}}\|\criticparam - \bar{\criticparam}\| < \lambda^+_{\max}(Q_{\bar{\criticparam}}) + L_{\bar{\criticparam}} \epsilon_{\criticparam}$.
    Therefore, $\jpolicyparam^*$ is locally exponentially stable for all $\tau \in (\tau^*, \infty)$ and all $\criticparam \in \region_{\bar{\criticparam}}$, where $\tau^* = \lambda^+_{\max}(Q_{\bar{\criticparam}}) + L_{\bar{\criticparam}} \epsilon_{\criticparam}$.
    The remainder of the proof follows that of~\cite[Theorem H.1]{fiez2021local} and classical results in stochastic approximation theory~\cite{borkar2009stochastic}.
    That is, there exists a neighborhood $\region_{\jpolicyparam^*}$ around $\jpolicyparam^*$ such that, from an initial point $\jpolicyparam_0 \in \region_{\jpolicyparam^*}$, sequence $\{\jpolicyparam_k\}$ converges to an internally chain transitive invariant set contained in $\region_{\jpolicyparam^*}$ almost surely for all $\criticparam \in \region_{\bar{\criticparam}}$, and the only such invariant set contained in $\region_{\jpolicyparam^*}$ is $\jpolicyparam^*$.
    \qed
\end{proof}

With \gls{tGDA}, we can not only construct a learning rate ratio $\tau^*$ to ensure local convergence near a \gls{DSE}/minimax equilibrium, but also find a $\tau_0$ such that a non-equilibrium critical point is unstable, and thereby can be avoided with arbitrarily small perturbation.
This is formalized in the following lemma, of which the proof resembles that of~\autoref{lem:tGDA_stab} and~\cite[Theorem~2]{fiez2021local}.

\begin{lemma}[Instability of Spurious Critical Points under \gls{tGDA}]
    \label{lem:tGDA_instab}
    Consider zero-sum game $(-\cost, \cost)$ parameterized by $\criticparam \in \region_{\bar{\criticparam}}$.
    If $\jpolicyparam^* := (\policycparam^*,\policydparam^*)$ is a critical point, \ie, $\grad_\policycparam \cost(\bar{\criticparam}, \jpolicyparam^*) = 0$, $\grad_\policydparam \cost(\bar{\criticparam}, \jpolicyparam^*) = 0$, $\det \grad^2_\policydparam \cost(\bar{\criticparam}, \jpolicyparam^*) \neq 0$, but not a \gls{DSE}, then there exists a $\tau_0 \in (0, \infty)$ such that, for all $\tau \in (\tau_0, \infty)$ and for all $\criticparam \in \region_{\bar{\criticparam}}$, $\jpolicyparam^*$ is an unstable equilibrium of the limiting system $\dot{\jpolicyparam} = (\grad_\policycparam \cost(\criticparam, \jpolicyparam), - \tau_a \grad_\policydparam \cost(\criticparam, \jpolicyparam))$.
\end{lemma}

Now, we are ready to state our main result, which shows that the \gls{MAGICS} procedure converges locally to a \gls{DSE} of game~\eqref{eq:trilevel}.

\begin{theorem}[Convergence of \gls{MAGICS}]
\label{thm:MAGICS}
    Consider the general-sum game $(\criticcost, -\cost, \cost)$ defined in~\eqref{eq:trilevel}.
    Let \autoref{assump:main} hold.
    If $(\criticparam^*,\policycparam^*,\policydparam^*)$ is a \gls{DSE} and $\lr^c_i = \smalloh(\lr^\ctrl_i)$,  there exists a $\tau_a^* \in (0, \infty)$  
    and a neighbourhood $\region$ around $(\criticparam^*,\policycparam^*,\policydparam^*)$ such that, for all $(\criticparam_0, \policycparam_0, \policydparam_0) \in U$ and all $\tau_a \in (\tau_a^*, \infty)$, iterates $(\criticparam_t, \policycparam_t, \policydparam_t)$, \ie, state of system~\eqref{eq:magics_dt_sys}, converge asymptotically almost surely to $(\criticparam^*,\policycparam^*,\policydparam^*)$.
\end{theorem}

\begin{proof}
    Since $(\criticparam^*,\policycparam^*,\policydparam^*)$ is a \gls{DSE}, by the implicit function theorem, there exists a neighbourhood $\region^1_{\criticparam^*}$ around $\criticparam^*$ and unique functions $\resmap_{\policycparam}$, $\resmap_{\policydparam}$ such that $\resmap_{\policycparam}(\criticparam^*) = \policycparam^*$, $\resmap_{\policydparam}(\criticparam^*) = \policydparam^*$, $\grad_\policycparam \cost(\criticparam, \resmap_{\policycparam}(\criticparam), \resmap_{\policydparam}(\criticparam)) = 0$, and $\grad_\policydparam \cost(\criticparam, \resmap_{\policycparam}(\criticparam), \resmap_{\policydparam}(\criticparam)) = 0$ for all $\criticparam \in \region^1_{\criticparam^*}$.
    Moreover, there exists a neighbourhood $\region^2_{\criticparam^*}$ around $\criticparam^*$ on which Hessians $\total_\criticparam^2\criticcost(\criticparam, \resmap_{\policycparam}(\criticparam), \resmap_{\policydparam}(\criticparam)) \succ 0$, $\grad^2_\policycparam\cost(\criticparam, \resmap_{\policycparam}(\criticparam), \resmap_{\policydparam}(\criticparam)) \prec 0$, ${\grad^2_\policydparam\cost(\criticparam, \resmap_{\policycparam}(\criticparam), \resmap_{\policydparam}(\criticparam)) \succ 0}$.
    Since $\grad^2_\policycparam\cost(\criticparam^*, \policycparam^*, \policydparam^*) \prec 0$ and $\grad^2_\policydparam\cost(\criticparam^*, \policycparam^*, \policydparam^*) \succ 0$, there exists a neighbourhood $\region_{\policycparam^*} \times \region_{\policydparam^*}$ around $(\policycparam^*, \policydparam^*)$ such that $\grad^2_\policycparam\cost(\criticparam, \policycparam, \policydparam) \prec 0$ and $\grad^2_\policydparam\cost(\criticparam, \policycparam, \policydparam) \succ 0$ for all $(\criticparam, \policycparam, \policydparam) \in \region_{\criticparam^*} \times \region_{\policycparam^*} \times \region_{\policydparam^*}$, where $\region_{\criticparam^*} \subseteq \region^1_{\criticparam^*} \cap \region^2_{\criticparam^*}$ is a non-empty open set.

    Next, we show that any $(\criticparam_0, \policycparam_0, \policydparam_0) \in \region \subseteq \region_{\criticparam^*} \times \region_{\policycparam^*} \times \region_{\policydparam^*}$ will converge asymptotically almost surely to DSE $(\criticparam^*, \policycparam^*, \policydparam^*)$.
    The continuous-time limiting system of~\eqref{eq:magics_dt_sys} is $(\dot{\criticparam}, \dot{\policycparam}, \dot{\policydparam}) = (-\total_\criticparam \criticcost(\criticparam, \policycparam, \policydparam), \grad_\policycparam \cost(\criticparam, \policycparam, \policydparam), - \tau_a  \grad_\policydparam \cost(\criticparam, \policycparam, \policydparam))$.
    Note that~\eqref{eq:magics_dt_sys:critic} can be written as $\criticparam_{t+1} = \criticparam_t - \lr^\ctrl_t \zeta_t$, where $\zeta_t = \frac{\lr^c_t}{\lr^\ctrl_t} \left(\total_\criticparam \criticcost(\criticparam_t, \policycparam_t, \policydparam_t) + \noise_{\criticparam,t} \right)$. Since $\zeta_t = \smalloh(1)$ for all $t = 0,1,\ldots$, the term is asymptotically negligible and discrete-time state $(\criticparam_t, \policycparam_t, \policydparam_t)$ tracks $(\dot{\criticparam} = 0, \dot{\policycparam} = \grad_\policycparam \cost(\criticparam, \policycparam, \policydparam), \dot{\policydparam} = - \tau_a  \grad_\policydparam \cost(\criticparam, \policycparam, \policydparam))$.
    By~\autoref{lem:tGDA_stab}, this system is locally exponentially stable in a region $\bar{\region}$ around the DSE $(\criticparam^*, \policycparam^*, \policydparam^*)$ for all $\tau_a \in (\tau_a^*, \infty)$.
    Therefore, there exists a local Lyapunov function on $\region = \bar{\region} \cap  \region_{\criticparam^*} \times \region_{\policycparam^*} \times \region_{\policydparam^*}$, which shows that the discrete-time state trajectory and continuous-time flow, both starting from any $(\criticparam_0, \policycparam_0, \policydparam_0) \in \region$, asymptotically contract onto each other.
    The remainder of the proof follows~\cite[Lemma G.2]{fiez2020implicit}, which shows that $\lim_{k \rightarrow \infty}\|(\criticparam_k, \policycparam_k, \policydparam_k) - (\criticparam^*, \policycparam^*, \policydparam^*)\| \rightarrow 0$ almost surely for all $(\criticparam_0, \policycparam_0, \policydparam_0) \in \region$.
    \qed
\end{proof}

While \autoref{thm:MAGICS} guarantees that \gls{MAGICS} can converge to a DSE when the network parameters are initialized in a local region around it, we can leverage \autoref{lem:tGDA_instab} to escape a non-DSE critical point.

\begin{proposition}
    Consider the general-sum game $(\criticcost, -\cost, \cost)$ defined in~\eqref{eq:trilevel}.
    If parameters $(\criticparam^*,\policycparam^*,\policydparam^*)$ is a critical point of the game but not a \gls{DSE},  there exists a $\tau_0 \in (0, \infty)$ such that, for all $\tau_a \in (\tau_0, \infty)$, $(\criticparam^*,\policycparam^*,\policydparam^*)$ is an unstable equilibrium of  the limiting system $(\dot{\criticparam}, \dot{\policycparam}, \dot{\policydparam}) = (-\total_\criticparam \criticcost(\criticparam, \policycparam, \policydparam), \grad_\policycparam \cost(\criticparam, \policycparam, \policydparam), - \tau_a  \grad_\policydparam \cost(\criticparam, \policycparam, \policydparam))$.
\end{proposition}

\section{Convergent Neural Synthesis of Robot Safety}
\label{sec:isaacs}

In this section, we apply \gls{MAGICS} to high-dimensional robot safety analysis, and propose the \gls{MAGICS}-Safety algorithm for convergent neural synthesis of safe robot policies.
Following prior work on \gls{RL}-based approximate reachability analysis~\cite{fisac2019bridging,hsunguyen2023isaacs}, we consider a time-discounted version of Isaacs equation~\eqref{eq:Isaacs_basic}:
\begin{equation}
\label{eq:Isaacs_discounted}
\begin{aligned}
    \valfunc^\shield(\state) = \min \left\{ (1-\discount) g(\state), \discount \max \Big\{ \ell(\state),  \max_{\ctrl \in \cset} \min_{\dstb \in \dset} \valfunc^\shield \big(\dyn (\state, \ctrl, \dstb)\big) \Big\}  \right\}.
\end{aligned}
\end{equation}
The discount factor $\discount \in (0,1)$ leads to a probabilistic interpretation, \ie, there is $(1-\discount)$ probability that the episode terminates immediately due to loss of safety.
Note that as $\discount \rightarrow 1$, we recover the undiscounted Isaacs equation~\eqref{eq:Isaacs_basic}.
Then, we apply \gls{SAC} to approximately solve \eqref{eq:Isaacs_discounted}.
The critic minimizes loss
\begin{equation}
    \label{eq:isaacs_critic}
    \hspace{-0.2cm}\criticcost^\shield(\criticparam,\policycparam,\policydparam) = \expectation_{\xi \sim \buffer} \left[ \left(\qfunc_{\criticparam_1} (\state,\ctrl,\dstb) - (1-\discount) g^\prime - \discount \min \left\{g^\prime, \qfunc_{\criticparam_2}(\state^\prime, \ctrl^\prime, \dstb^\prime) \right\} \right)^2\right],
\end{equation}
where $g^\prime := g(\state^\prime)$, $\xi = (\state,\ctrl,\dstb,g^\prime,\state^\prime)$, $\ctrl^\prime \sim \policyc_{\policycparam} (\cdot | \state^\prime)$, and $\dstb^\prime \sim \policyd_{\policydparam} (\cdot | \state^\prime)$.
The controller maximizes objective
\begin{equation}
    \label{eq:isaacs_actor}
    \cost^\shield(\criticparam,\policycparam,\policydparam) = \expectation_{\state \sim \buffer} \left[ \qfunc_{\criticparam_1} (\state,\tilde{\ctrl},\tilde{\dstb})
    - \entreg^\ctrl \log \policyc_\policycparam(\tilde{\ctrl} | \state)
    + \entreg^\dstb \log \policyd_\policydparam(\tilde{\dstb} | \state)\right],
\end{equation}
where $\tilde{\ctrl} \sim \policyc_{\policycparam} (\cdot | \state)$, $\tilde{\dstb} \sim \policyd_{\policydparam} (\cdot | \state)$.
The disturbance minimizes $\cost^\shield(\criticparam,\policycparam,\policydparam)$.

The \gls{MAGICS}-Safety training objectives \eqref{eq:isaacs_critic}-\eqref{eq:isaacs_actor} differ from~\cite{hsunguyen2023isaacs} in that they explicitly capture the coupling among the critic and actors (hence their non-cooperative interactions).
This game-theoretic formulation (c.f.~\eqref{eq:trilevel}) of adversarial \gls{RL} facilitates developing convergence guarantees using the \gls{MAGICS} paradigm in \autoref{sec:approach}.
In the following theorem, we show that \gls{MAGICS}-Safety locally converges to a \gls{DSE}, which is a direct consequence of \gls{MAGICS} convergence in \autoref{thm:MAGICS}.

\begin{figure}[!hbtp]
  \centering
  \includegraphics[width=0.7\columnwidth]{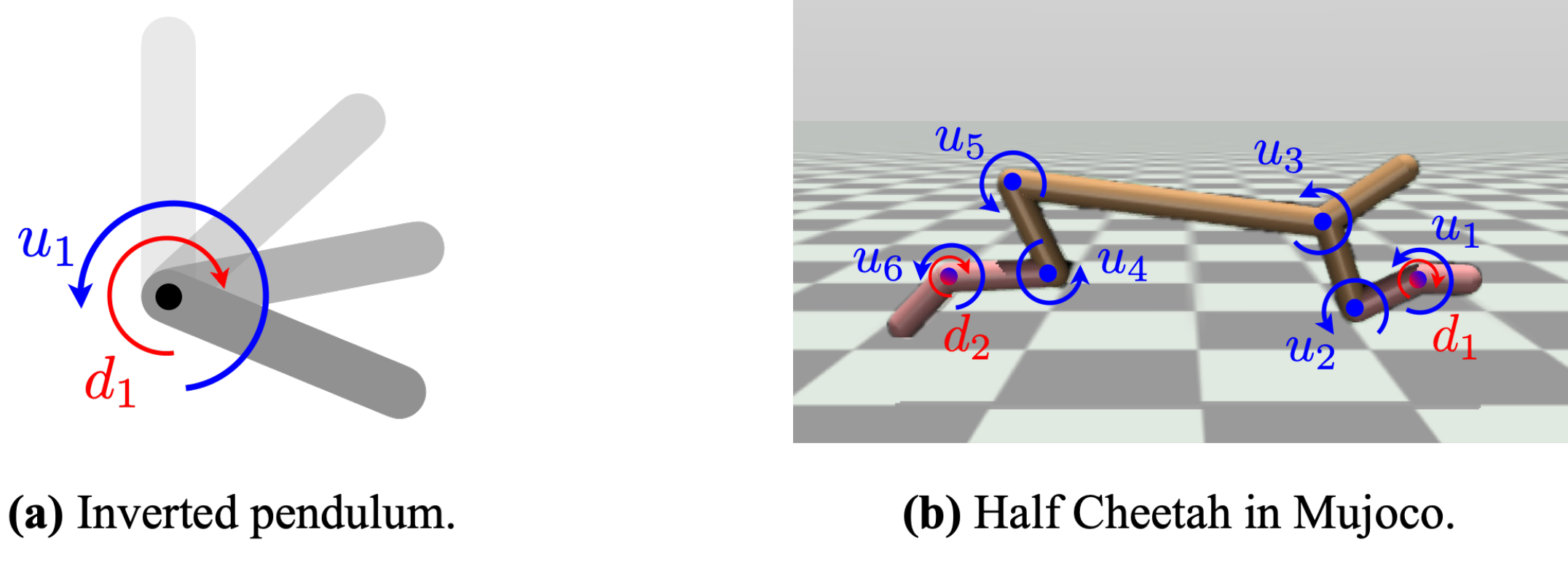}
  \caption{\label{fig:gym_examples} The Pendulum and Half Cheetah environments in OpenAI Gym~\cite{brockman2016openai}, with control and disturbance actions represented in blue and red arrows, respectively.
  The Pendulum's control inputs are one-dimensional torques applied to the end of the rod in opposition to each other.
  The Half Cheetah has a six-dimensional control input on the notated joints, with the disturbance acting to destabilize the cheetah through additional torques on its paws. 
  \vspace{-5mm}
  }
\end{figure}

\begin{theorem}[Convergence of \gls{MAGICS}-Safety]
    Consider the general-sum game defined in~\eqref{eq:trilevel} with game objectives $(\criticcost^\shield, -\cost^\shield, \cost^\shield)$ defined in \eqref{eq:isaacs_critic}-\eqref{eq:isaacs_actor}.
    Let \autoref{assump:main} hold.
    If $(\criticparam^*,\policycparam^*,\policydparam^*)$ is a \gls{DSE} and $\lr^c_i = \smalloh(\lr^\dstb_i)$,  there exists a $\tau_a^* \in (0, \infty)$  
    and a neighbourhood $\region^\shield = \region_{\criticparam^*}^\shield \times \region_{\policycparam^*}^\shield \times \region_{\policydparam^*}^\shield$ around $(\criticparam^*,\policycparam^*,\policydparam^*)$ such that, for all $(\criticparam_0, \policycparam_0, \policydparam_0) \in U^\shield$ and all $\tau_a \in (\tau_a^*, \infty)$, iterates $(\criticparam_t, \policycparam_t, \policydparam_t)$ of the \gls{MAGICS}-Safety training procedure converge asymptotically almost surely to $(\criticparam^*,\policycparam^*,\policydparam^*)$.
\end{theorem}

\vspace{-2mm}
\section{Experiments}
\vspace{-2mm}
In this section, we illustrate the strength of our approach in two simulated and one hardware examples that differ in task, problem scale, and computation approach.
\textit{Our main hypothesis is that the Stackelberg-minimax robust \gls{RL} is likelier to yield a stronger policy (for both controller and disturbance) than non-game robust \gls{RL} baselines.
}

\vspace{-3mm}
\subsection{Simulated Examples: Robust Control in OpenAI Gym}
Our simulation experiments build upon the OpenAI~\cite{brockman2016openai} Gym with the Mujoco physics simulator~\cite{todorov2012mujoco} and Stable-Baselines3~\cite{sb3} platforms. Specifically, we adapt the Pendulum and Half Cheetah environments (\autoref{fig:gym_examples}), originally created for single-agent \gls{RL}, to account for two inputs--controller and disturbance.
\begin{figure}[!hbtp]
  \centering
  \includegraphics[width=1.0\columnwidth]{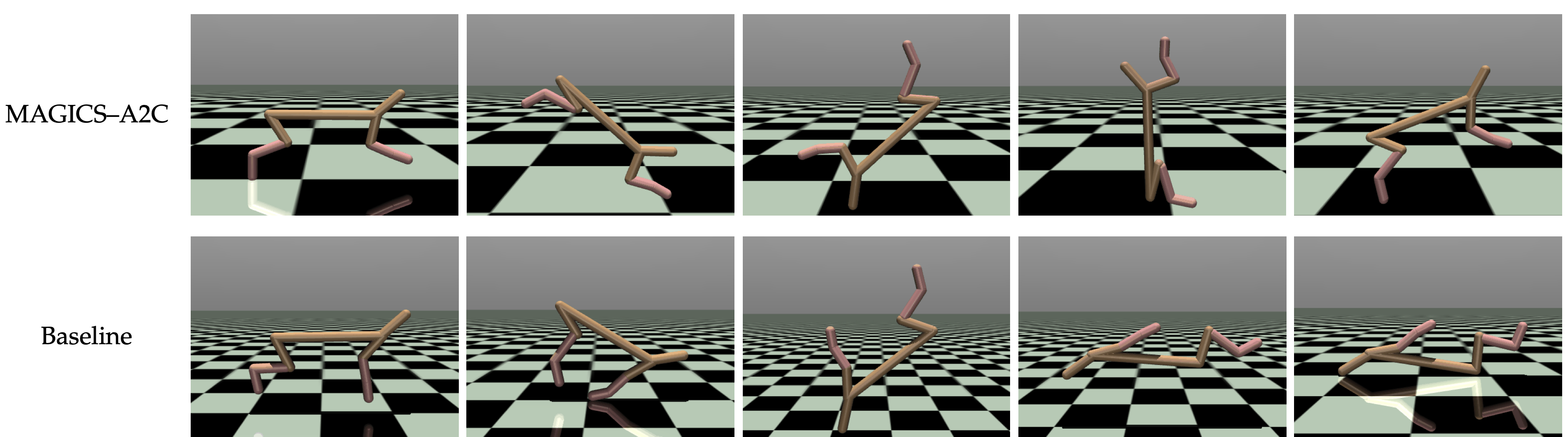}
  \caption{\label{fig:halfcheetah} Snapshots of the Half Cheetah controlled by \gls{MAGICS}-\gls{A2C} and baseline-\gls{A2C}. 
  Despite an excessively large
  disturbance torque,
  the \gls{MAGICS}-\gls{A2C} policy manage to flip the robot back upright and resumed normal gaits.
  In contrast, the baseline-\gls{A2C} policy is unable to recover the robot from the overturn; it moved awkwardly on its face and back, wiggling its feet.
  \vspace{-5mm}
  }
\end{figure}
We compare MAGICS-\gls{A2C} and MAGICS-\gls{SAC} against an ablation study and a baseline method---the ablation removes the critic's Stackelberg gradient but preserves $\tau$-GDA between the agents, and the baseline enforces all learning rates to be identical.
\autoref{tab:pend_confusion_matrix} and \autoref{tab:pend_confusion_matrix:SAC} display results from a series of round-robin matches between the learned control and disturbance policies for all three methods. For each match, the selected control and disturbance play five sets of 100 games. We report the win rate of each controller for each controller/disturbance strategy pair.
We assign a controller failure should any of the following occur: (1) the controller swings the pendulum into the upright position, but the disturbance is able to destabilize it (the pendulum tip moves outside $\pm$10$^\circ$ from the vertical), (2) the controller swings the pendulum up, but cannot swing it soon enough for more than $5$ seconds to remain, or (3) the controller completely fails to bring the pendulum to the upright position.
If none of these three failure modes occur, we consider the controller to have won the game.
In these tables, for each disturbance strategy, we highlight, in bold-type, the most performant control strategy. All hyperparameters and model architecture (\eg, number of hidden layers, activation functions) among the models are identical with the exception of the learning rates as per $\tau$-GDA.

For the pendulum, our results demonstrate that the \gls{MAGICS} controller outperforms the ablation and baseline controllers against all three disturbances. In addition, since the two actors evolve together, there is strong evidence to believe that, for the snapshot of performance depicted in the tables, each disturbance relative to its associated control, is also performant---the diagonal entries depict a win rate that is close to even.
However, given that \gls{MAGICS}'s disturbance can resoundingly defeat the other individually strong control strategies, there is strong evidence that the critic's Stackelberg gradient encourages the disturbance to learn alternative strategies to attack the control.

\begin{figure}
    \centering
    \includegraphics[width=0.85\linewidth]{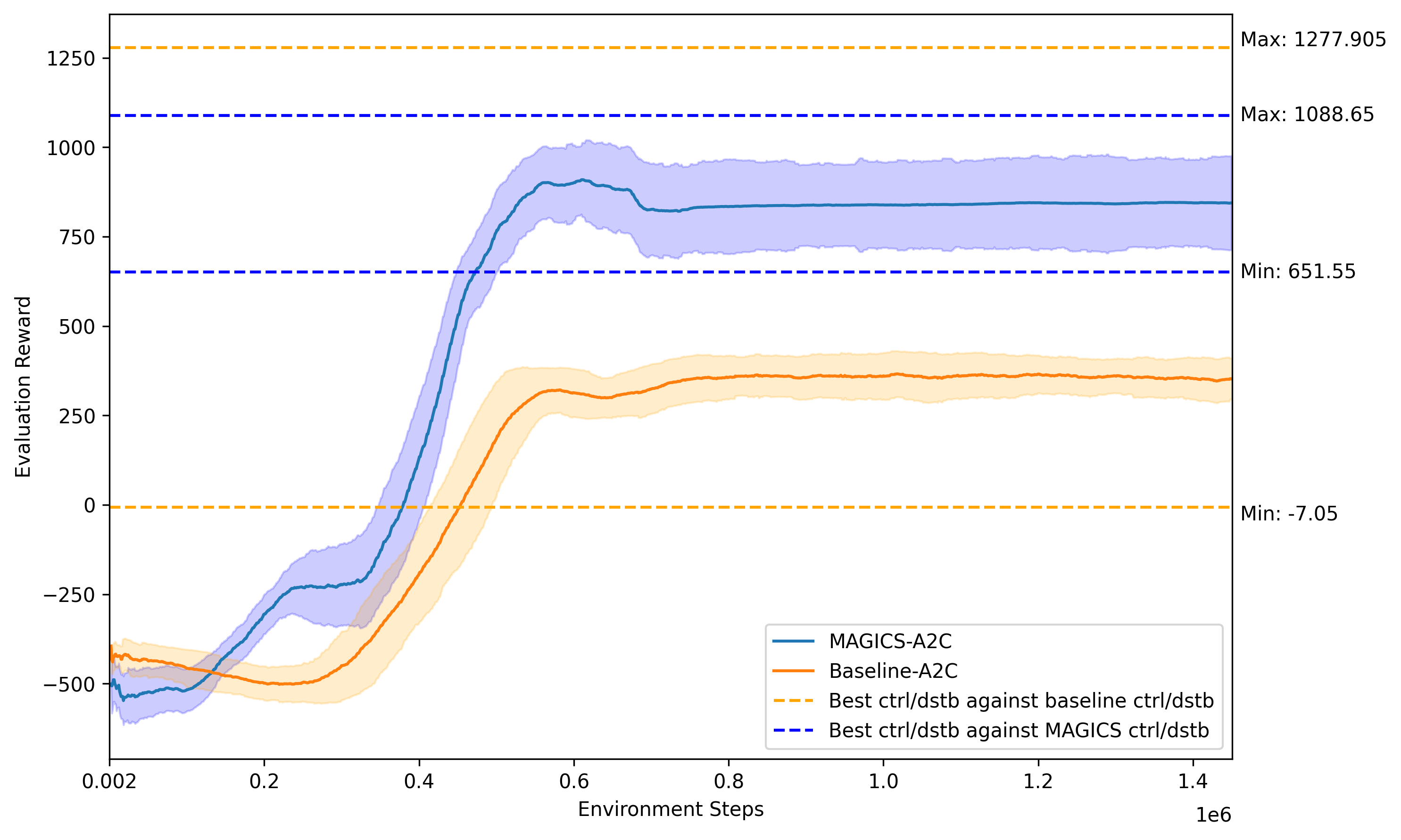}
    \caption{Cumulative reward curves across five seeds of MAGICS-\gls{A2C} (blue) and baseline-\gls{A2C} (orange) for the adversarial Half Cheetah environment. MAGICS-\gls{A2C} converges to an equilibrium that outperforms the converged baseline equilibrium by $\sim2.7$ times. Dashed lines represent exploiter disturbances against the same controller color.
    \vspace{-6mm}}
    \label{fig:halfcheetah_rew_curves}
\end{figure}

We also perform an analogous tournament in the Half Cheetah example. However, for this environment, success or failure can no longer be straightforwardly defined with a single logical variable.
For the Half Cheetah, the objective is to run to the right as quickly and for as long as possible (until the episode is terminated at a maximum step count).
To this end, we report the raw cumulative reward in \autoref{tab:cheetah_confusion_matrix_a2c:rwd}.
MAGICS-\gls{A2C} achieves significantly higher reward compared to both ablation and baselines. Furthermore, the MAGICS-\gls{A2C} disturbance is quantitatively stronger than that of the baseline, as the controller reward increases when MAGICS-\gls{A2C}'s controller is pitted against the baseline disturbance. \autoref{fig:halfcheetah} showcases the robustness of the \gls{MAGICS} controller with a representative example.
The baseline controller fails to learn how to recover the cheetah from turnover. In contrast, 
\gls{MAGICS}'s controller has learned to leverage the momentum---when the disturbance tips the cheetah too far forward, the controller adapts the fall into a roll---to get the cheetah back to its feet.

\setlength\tabcolsep{10pt}
\begin{table}[!hbtp]
\centering
\resizebox{\columnwidth}{!}{
\begin{tabular}{l|cccc}
\toprule
\diagbox{ctrl. strategy}{dstb. strategy} &
\gls{MAGICS}-A2C &
\gls{MAGICS}-A2C-ablation &
Baseline-A2C  &
\\ \midrule
\gls{MAGICS}-A2C  $\uparrow$                  &$\mathbf{78.8\%}$             &$\mathbf{67.4\%}$            &$\mathbf{74.6\%}$     \\
\gls{MAGICS}-A2C-ablation   $\uparrow$              &$36.4\%$                      &$40.6\%$                     &$38.2\%$     \\
Baseline-A2C    $\uparrow$            &$59.6\%$                      &$58.0\%$                     &$53.0\%$     \\ \bottomrule
\end{tabular}
}
\vspace{0.1em}
\caption{Win rates across five sets of 100 zero-sum games between the corresponding controller and disturbance from a random initial state in the Pendulum environment. Stackelberg gradient boosts the self-play win rate by approximately 38\%, as evidenced by the diagonal entries.
\vspace{-7mm}
}
\label{tab:pend_confusion_matrix}
\end{table}
\vspace{-8mm}

\setlength\tabcolsep{10pt}
    \begin{table}[!hbtp]
    \centering
    \resizebox{\columnwidth}{!}{
    \begin{tabular}{l|cccc}
    \toprule
    \diagbox{ctrl. strategy}{dstb. strategy} &
    \gls{MAGICS}-SAC &
    \gls{MAGICS}-SAC-ablation &
    Baseline-SAC  &
    \\ \midrule
    \gls{MAGICS}-SAC  $\uparrow$                  &$\mathbf{66.5\%}$             &$\mathbf{63.25\%}$            &$80.75\%$     \\
    \gls{MAGICS}-SAC-ablation   $\uparrow$              &$51.5\%$                      &$57.5\%$          &$\mathbf{83.0\%}$     \\
    Baseline-SAC    $\uparrow$            &$61.25\%$               &$60.25\%$                     &$80.25\%$     \\ 
    \bottomrule
    \end{tabular}
    }
    \vspace{0.1em}
    \caption{Win rates obtained from 100 zero-sum games between the corresponding five controller and disturbance random seeds for the Pendulum. MAGICS-\gls{SAC} outperforms almost all the other disturbance classes.
    \vspace{-15mm}
    }
    \label{tab:pend_confusion_matrix:SAC}
    \end{table}

\setlength\tabcolsep{10pt}
    \begin{table}[!hbtp]
    \centering
    \resizebox{\columnwidth}{!}{
    \begin{tabular}{l|cccc}
    \toprule
    \diagbox{ctrl. strategy}{dstb. strategy} &
    \gls{MAGICS}-A2C &
    \gls{MAGICS}-A2C-ablation &
    Baseline-A2C  &
    \\ \midrule
    \gls{MAGICS}-A2C  $\uparrow$                  &$\mathbf{972.14 \pm 289.19}$             &$\mathbf{967.78 \pm 376.77}$            &$\mathbf{1121.81 \pm 137.13}$     \\
    \gls{MAGICS}-A2C-ablation   $\uparrow$              &$316.46 \pm 281.73$                      &$285.41 \pm 275.00$                     &$325.56 \pm 277.35$     \\
    Baseline-A2C    $\uparrow$            &$360.85 \pm 154.86$                      &$335.57 \pm 115.96$                     &$409.04 \pm 118.44$     \\ \bottomrule
    \end{tabular}
    }
    \vspace{0.1em}
    \caption{Episodic reward obtained from five sets of 100 zero-sum games between the corresponding controller and disturbance from a random initial state for the Half Cheetah. MAGICS-\gls{A2C} achieves \emph{significantly} higher reward against all other disturbance classes, with qualitatively more innovative strategies found relative to that of the baseline (\autoref{fig:halfcheetah}).
    \vspace{-15mm}
    }
    \label{tab:cheetah_confusion_matrix_a2c:rwd}
    \end{table}

\setlength\tabcolsep{10pt}
    \begin{table}[!hbtp]
    \centering
    \resizebox{\columnwidth}{!}{
    \begin{tabular}{l|cccc}
    \toprule
    \diagbox{ctrl. strategy}{dstb. strategy} &
    \gls{MAGICS}-SAC &
    \gls{MAGICS}-SAC-ablation &
    Baseline-SAC  &
    \\ \midrule
    \gls{MAGICS}-SAC  $\uparrow$                  &$\mathbf{1443.84 \pm 143.26}$             &$\mathbf{1458.82 \pm 145.35}$            &$\mathbf{1557.08 \pm 114.35}$     \\
    \gls{MAGICS}-SAC-ablation   $\uparrow$              &$1218.30 \pm 52.44$                      &$1206.81 \pm 75.45$                     &$1228.61 \pm 52.06$     \\
    Baseline-SAC    $\uparrow$            &$990.07 \pm 342.85$                      &$1016.99 \pm 326.30$                     &$1020.71 \pm 321.99$     \\ \bottomrule
    \end{tabular}
    }
    \vspace{0.1em}
    \caption{Episodic reward obtained from five sets of 100 zero-sum games between the corresponding controller and disturbance from a random initial state for the Half Cheetah. SAC, being off-policy and state-of-the-art for model-free RL, achieves a much higher reward in all domains compared to A2C. And yet, MAGICS-\gls{SAC} persistently achieves a higher reward against all other disturbance classes.
    \vspace{-5mm}
    }
    \label{tab:cheetah_confusion_matrix_sac:rwd}
    \end{table}

\begin{figure}[!hbtp]
  \centering
  \includegraphics[width=1.0\columnwidth]{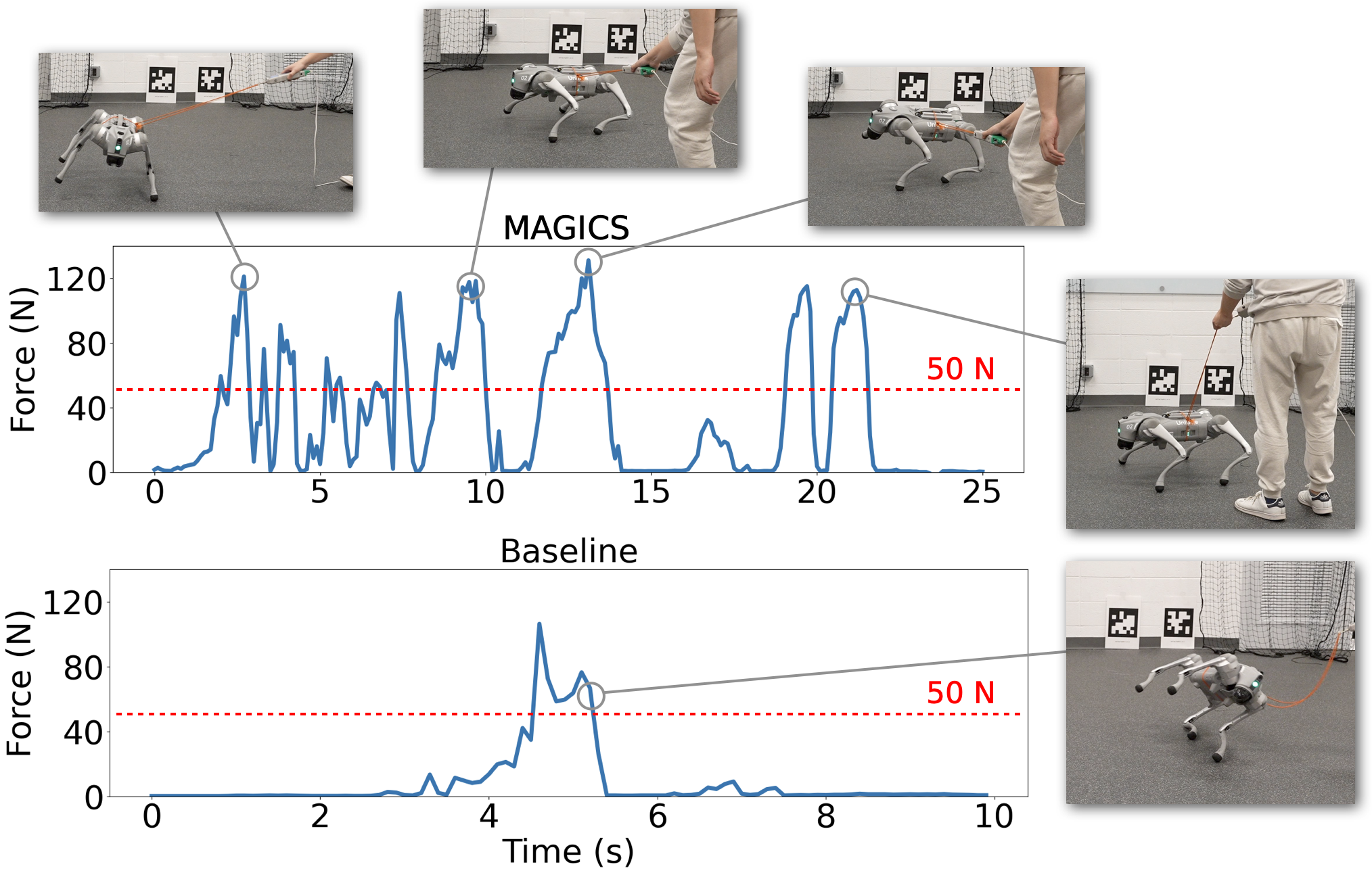}
  \caption{\label{fig:force_plot} Time evolution of the human's tugging forces (disturbance) with \gls{MAGICS}-Safety and the baseline.
  Both policies are trained in simulation with a maximum of 50 N tugging force disturbance.
  The \gls{MAGICS}-Safety policy is robust against the varying tugging forces from different angles, while the baseline failed even with tugging forces of smaller magnitude.
  \vspace{-5mm}
  }
\end{figure}

We also display reward curves for the Half Cheetah environment in \autoref{fig:halfcheetah_rew_curves}. From the same initial conditions, qualitatively, \gls{MAGICS} converges to more performant equilibria than the baseline does, with a controller reward approximately $\sim2.7$ times greater than that of the baseline controller and disturbance pair. We also plot, in dashed lines, the final exploit curve results: the ``best'' and ``worst'' disturbances against the ``best'' controllers of that color, \eg, dashed-blue lines represent the strongest MAGICS-\gls{A2C} controller against the strongest and weakest of the baseline-\gls{A2C} disturbances. Even against the strongest baseline-trained disturbance, MAGICS-\gls{A2C}'s performance only degrades slightly, as opposed to the baseline controller against the \gls{MAGICS} disturbance, in which all strategies fail as evidenced by the extremely poor reward.
\vspace{-3mm}
\subsection{Hardware Demonstration: Safe Quadrupedal Locomotion}

Next, we apply \gls{MAGICS}-Safety (\autoref{sec:isaacs}) to a safe quadrupedal locomotion task.

\p{Hardware} 
We utilize the Unitree Go-2 as the test robot platform.
The robot is equipped with built-in IMUs for obtaining measurements about body angular velocities and linear acceleration, and internal motor encoders that measure joint positions and velocities.
The robot also provides a Boolean contact signal for each foot.
No visual perception is used for computing the control policy.

\p{System Dynamics}
The quadrupedal robot's state is 36-dimensional, including positions of the body frame, velocities of the robot’s torso, rotation angles, axial rotational rates, angle, angular velocity, and commanded angular increment of the robot's joints.
The robot’s control inputs include independent torques applied on each of its 12 rotational joints provided by an electric motor.
Our neural control policy sends a reference signal to each motor, which is tracked by a low-level controller.
See~\cite{nguyen2024gameplay} for a detailed explanation of the robot's dynamic model and safety specification.

\p{Baseline} 
We compare to ISAACS~\cite{hsunguyen2023isaacs,nguyen2024gameplay}, the former state-of-the-art adversarial neural safety synthesis method, which uses individual gradients and alternating optimization for updating the critic and actor parameters.
Both policies are trained with the PyBullet~\cite{coumans2016pybullet} physics engine.
We use the same neural network architecture as~\cite{nguyen2024gameplay}.

\p{Policy Training} 
We train both policies for 1.5 million steps.
During training, we use a maximum of 50 N tugging force disturbance for both policies.
Once the policies are trained offline, they are deployed online within a \textit{value-based} safety filter~\cite[Sec.~3.2]{hsu2023sf} to prevent the robot from falling under adversarial human tugging forces.

\p{Evaluation: The Tugging Experiment} 
We evaluate the robustness of the trained policies with a human tugging the robot from different angles and using forces with a varying magnitude (\autoref{fig:front_fig}).
The time evolution of the forces applied by the human is plotted in \autoref{fig:force_plot}.
The \gls{MAGICS}-Safety policy is able to withstand the external force, with a peak value above 120 N, for the entire test horizon of $25$ seconds.
The baseline fails to resist a force with a peak value less than 120 N.
This test empirically demonstrates the superior robustness of the \gls{MAGICS}-Safety policy compared to the baseline.

\vspace{-2mm}
\section{Limitations and Future Work}
Our proposed \gls{MAGICS} algorithm requires the computation of second-order information over the space of neural network parameters (c.f.~\eqref{eq:critic_total_deriv_compact}), which can be expensive to obtain.
Recent advances in Stackelberg learning using only first-order information~\cite{maheshwari2023convergent} offer a promising pathway to ease the computation burden.
In this paper, we focus exclusively on developing a game-theoretic variant of the 
actor--critic \gls{RL}.
Our algorithmic and theoretical framework may be extended to broader multi-agent \gls{RL} settings with different algorithms, such as policy gradient, and with general-sum objectives.

\vspace{-2mm}
\section{Conclusions}
In this paper, we introduced Minimax Actors Guided by Implicit Critic Stackelberg (MAGICS), a novel game-theoretic reinforcement learning algorithm that is provably convergent to an equilibrium solution.
Building on \gls{MAGICS}, we also offered convergence assurances for an \gls{RL}-based robot safety synthesis method.
Our empirical evaluations, conducted through simulations in OpenAI Gym and hardware tests using a 36-dimensional quadruped robot, demonstrated that \gls{MAGICS} produced robust control policies consistently outperforming the state-of-the-art neural safe control method.

\vspace{-3mm}

\section*{Acknowledgements}
The authors thank Chi Jin, Wenzhe Li, Zixu Zhang, Kai-Chieh Hsu, and Kaiqu Liang for helpful discussions throughout the project. Compute time was provided by Princeton Research Computing HPC facilities. The work has been partly supported by the DARPA LINC program and an NSF CAREER Award.

\bibliography{references}
\bibliographystyle{splncs04}

\end{document}

%% file: headers/commands.tex
\usepackage{multicol}
\usepackage[bookmarks=true]{hyperref}

\usepackage{float}
\usepackage{amsfonts}
\usepackage{comment}
\usepackage{times}
\usepackage{amsmath}
\usepackage{changepage}
\usepackage{amssymb}
\usepackage{graphicx}
\usepackage{adjustbox}
\usepackage{latexsym}
\usepackage{enumerate}
\usepackage[font=small,belowskip=0pt]{caption}
\usepackage{mathtools}
\usepackage{algcompatible}
\usepackage{algpseudocode}
\usepackage[ruled,vlined,linesnumbered]{algorithm2e}
\usepackage{bbm}
\usepackage{xcolor}
\usepackage{bm}
\usepackage{booktabs}
\usepackage{comment}
\usepackage[xindy, acronyms, nowarn]{glossaries}   %

\usepackage{fontawesome5}

\definecolor{porange}{HTML}{E77500} %

\SetCommentSty{mycommfont}

\usepackage{url}
\usepackage{hyperref}
 \hypersetup{
     colorlinks=true,
     linkcolor=blue,
     filecolor=blue,
     citecolor=blue,      
     urlcolor=black,
     }
\usepackage{diagbox}

%% file: headers/notation.tex
\usepackage{amsmath}
\usepackage{bm}

\newtheorem{assumption}{Assumption}

\AtBeginEnvironment{proof}{\footnotesize}

\usepackage{ifthen}
\newboolean{include-notes}
\newboolean{include-new}
\newboolean{include-remove}
\setboolean{include-notes}{true}
\setboolean{include-new}{true}
\setboolean{include-remove}{false}

\usepackage{xcolor}
\usepackage[normalem]{ulem}
\newcommand{\jaime}[1]{\ifthenelse{\boolean{include-notes}}{\textcolor{orange}{\textbf{Jaime:} #1}}{}}
\newcommand{\haimin}[1]{\ifthenelse{\boolean{include-notes}}{\textcolor{teal}{\textbf{Haimin:} #1}}{}}
\newcommand{\justin}[1]{\ifthenelse{\boolean{include-notes}}{\textcolor{purple}{\textbf{Justin:} #1}}{}}
\newcommand{\remove}[1]{\ifthenelse{\boolean{include-remove}}{\textcolor{red}{\sout{#1}}}{}}
\newcommand{\new}[1]{\ifthenelse{\boolean{include-new}}{\textcolor{blue}{#1}}{#1}}
\renewcommand{\todo}[1]{\ifthenelse{\boolean{include-notes}}{\textcolor{magenta}{#1}}{#1}}

\usepackage{lipsum}

\newglossarystyle{compactlongstyle}{
    \setglossarystyle{long}

    \renewcommand*{\glsgroupheading}[1]{}
    
}

\glsdisablehyper
\makeglossaries

\newglossaryentry{LSE}
{
  name={LSE},
  plural={LSEs},
  description={local Stackelberg equilibrium},
  first={local Stackelberg equilibrium (\glsentrytext{LSE})},
  descriptionplural={local Stackelberg equilibria},
  firstplural={local Stackelberg equilibria (\glsentryplural{LSE})}
}

\newglossaryentry{FSE}
{
  name={FSE},
  plural={FSEs},
  description={feedback Stackelberg equilibrium},
  first={feedback Stackelberg equilibrium (\glsentrytext{FSE})},
  descriptionplural={feedback Stackelberg equilibria},
  firstplural={feedback Stackelberg equilibria (\glsentryplural{FSE})}
}

\newglossaryentry{DSE}
{
  name={DSE},
  plural={DSEs},
  description={differential Stackelberg equilibrium},
  first={differential Stackelberg equilibrium (\glsentrytext{DSE})},
  descriptionplural={differential Stackelberg equilibria},
  firstplural={differential Stackelberg equilibria (\glsentryplural{DSE})}
}

\newglossaryentry{RL}
{
  name={RL},
  description={reinforcement learning},
  first={reinforcement learning (\glsentrytext{RL})},
}

\newglossaryentry{HJI}
{
  name={HJI},
  description={Hamilton--Jacobi--Isaacs},
  first={Hamilton--Jacobi--Isaacs (\glsentrytext{HJI})},
}

\newglossaryentry{GAS}
{
  name={GAS},
  description={globally asymptotically stable},
  first={globally asymptotically stable (\glsentrytext{GAS})},
}

\newglossaryentry{a.s.}
{
  name={a.s.},
  description={almost surely},
  first={almost surely (\glsentrytext{a.s.})},
}

\newglossaryentry{MDP}
{
  name={MDP},
  description={Markov decision process},
  first={Markov decision process (\glsentrytext{MDP})},
}

\newglossaryentry{SAC}
{
  name={SAC},
  description={Soft Actor--Critic},
  first={Soft Actor--Critic (\glsentrytext{SAC})},
}

\newglossaryentry{tGDA}
{
  name={$\tau$-GDA},
  description={gradient descent-ascent with finite timescale separation},
  first={gradient descent-ascent with finite timescale separation (\glsentrytext{tGDA})},
}

\newglossaryentry{MAGICS}
{
  name={MAGICS},
  description={Minimax Actors Guided by Implicit Critic Stackelberg},
  first={Minimax Actors Guided by Implicit Critic Stackelberg (\glsentrytext{MAGICS})},
}

\newglossaryentry{ISAACS}
{
  name={ISAACS},
  description={Iterative Soft Adversarial Actor--Critic for Safety},
  first={Iterative Soft Adversarial Actor--Critic for Safety (\glsentrytext{ISAACS})},
}

\newglossaryentry{ISAACStbg}
{
  name={ISAAC$\mathcal{S}$},
  description={Iterative Soft Adversarial Actor--Critic with Stackelberg learning dynamics for Safety},
  first={Iterative Soft Adversarial Actor--Critic with Stackelberg learning dynamics for Safety (\glsentrytext{ISAACStbg})},
}

\newglossaryentry{A2C}
{
  name={A2C},
  description={Advantage Actor--Critic},
  first={Advantage Actor-Critic (\glsentrytext{A2C})},
}

\newglossaryentry{iHvp}
{
  name={iHvp},
  description={inverse hessian vector product},
  first={inverse-Hessian-vector product
  (\glsentrytext{iHvp})},
}

\newglossaryentry{jvp}
{
  name={jvp},
  description={jacobian vector product},
  first={Jacobian-vector product
  (\glsentrytext{jvp})},
}

\newcommand{\p}[1]{\smallskip \noindent \textbf{{#1}.}}

\newcommand{\eg}{\emph{e.g.}}
\newcommand{\ie}{\emph{i.e.}}

\newcommand{\reals}{\mathbb{R}}

\newcommand{\region}{{U}}

\newcommand{\distr}{p}

\DeclareMathOperator*{\expectation}{\mathbb{E}}

\newcommand{\grad}{{\nabla}}
\newcommand{\total}{{D}}

\newcommand{\jacobian}{\mathbf{J}}

\DeclareMathOperator*{\argmax}{arg\,max}
\DeclareMathOperator*{\argmin}{arg\,min}

\newcommand{\schur}{\operatorname{\mathtt{S}_1}}

\newcommand{\state}{{x}}
\newcommand{\ctrl}{{u}}
\newcommand{\dstb}{{d}}
\newcommand{\noise}{{v}}

\newcommand{\traj}{{\mathbf{x}}}%

\newcommand{\xset}{{\mathcal{X}}}  %
\newcommand{\cset}{{\mathcal{U}}}
\newcommand{\dset}{{\mathcal{D}}}

\newcommand{\nx}{{n}}
\newcommand{\nc}{{n_\ctrl}}
\newcommand{\nd}{{n_\dstb}}

\newcommand{\dyn}{{f}}

\newcommand{\valfunc}{{V}}
\newcommand{\qfunc}{{Q}}
\newcommand{\cost}{{J}}
\newcommand{\criticcost}{{L}}

\newcommand{\policy}{{\pi}}

\newcommand{\policyc}{\policy^\ctrl}
\newcommand{\policyd}{\policy^\dstb}

\newcommand{\policycparam}{{\theta}}
\newcommand{\policycparameq}{{\policycparam^*}}
\newcommand{\policydparam}{{\psi}}
\newcommand{\policydparameq}{{\policydparam^*}}
\newcommand{\criticparam}{{\omega}}
\newcommand{\jpolicyparam}{{y}}

\newcommand{\resmap}{r}

\newcommand{\reward}{r}

\newcommand{\discount}{\gamma}
\newcommand{\buffer}{\mathcal{B}}

\newcommand{\entreg}{\eta}
\newcommand{\lr}{\alpha}

\newcommand{\target}{{\mathcal{T}}}

\renewcommand{\failure}{{\mathcal{F}}}

\newcommand{\reachavoid}{{\mathcal{RA}}}

\newcommand{\shield}{\text{\tiny{\faShield*}}}

\newcommand{\smalloh}{{o}}